\documentclass{article}

\usepackage{microtype}
\usepackage{subfigure}
\usepackage{booktabs} 

\usepackage{hyperref}

\usepackage[accepted]{icml2025}

\usepackage{amsmath}
\usepackage{amssymb}
\usepackage{mathtools}
\usepackage{amsthm}

\usepackage{multirow}

\makeatletter
\newenvironment{theorem-repeat}[1]{%
    \begingroup
    \renewcommand{\thetheorem}{\ref{#1}}
    \addtocounter{theorem}{-1}%
    \begin{theorem}
}{%
    \end{theorem}
    \endgroup
}
\makeatother

\usepackage[capitalize,noabbrev]{cleveref}

\theoremstyle{plain}
\newtheorem{theorem}{Theorem}
\newtheorem*{theorem*}{Theorem}
\newtheorem{proposition}{Proposition}[section]
\newtheorem{lemma}{Lemma}[section]

\theoremstyle{definition}
\newtheorem{definition}{Definition}[section]

\theoremstyle{remark}
\newtheorem{condition}{Condition}

\usepackage[textsize=tiny]{todonotes}

\icmltitlerunning{Higher-order Structure Boosts Link Prediction on Temporal Graphs}

\begin{document}

\twocolumn[
\icmltitle{Higher-order Structure Boosts Link Prediction on Temporal Graphs}



\icmlsetsymbol{equal}{*}

\begin{icmlauthorlist}
\icmlauthor{Jingzhe Liu}{msu}
\icmlauthor{Zhigang Hua}{comp}
\icmlauthor{Yan Xie}{comp}
\icmlauthor{Bingheng Li}{msu}
\icmlauthor{Harry Shomer}{msu}
\icmlauthor{Yu Song}{msu}
\icmlauthor{Kaveh Hassani}{comp}
\icmlauthor{Jiliang Tang}{msu}
\end{icmlauthorlist}

\icmlaffiliation{msu}{Department of Computer Science and Engineering, Michigan State University}
\icmlaffiliation{comp}{Meta}



\vskip 0.3in
]
\printAffiliationsAndNotice{}



\begin{abstract}
Temporal Graph Neural Networks (TGNNs) have gained growing attention for modeling and predicting structures in temporal graphs. However, existing TGNNs primarily focus on pairwise interactions while overlooking higher-order structures that are integral to link formation and evolution in real-world temporal graphs. 
Meanwhile, these models often suffer from efficiency bottlenecks, further limiting their expressive power. 
To tackle these challenges, we propose a Higher-order structure Temporal Graph Neural Network (HTGN), which incorporates hypergraph representations into temporal graph learning. 
In particular, we develop an algorithm to identify the underlying higher-order structures, enhancing the model’s ability to capture the group interactions. 
Furthermore, by aggregating multiple edge features into hyperedge representations, HTGN effectively reduces memory cost during training. 
We theoretically demonstrate the enhanced expressiveness of our approach and validate its effectiveness and efficiency through extensive experiments on various real-world temporal graphs. 
Experimental results show that HTGN achieves superior performance on dynamic link prediction while reducing memory costs by up to 50\% compared to existing methods. 
\end{abstract}

\section{Introduction}

Graphs have been widely used to model interactions in complex real-world systems. 
With the rapid progress of deep learning, graph neural networks (GNNs) have been developed to predict properties at graph, node, or edge level~\cite{ma2021deep}. 
Recently, there has been increasing attention to graphs that evolve over time, such as transportation, transaction, or social networks, leading to the study of temporal graphs (TGs)~\cite{holme2012temporal, masuda2016guide}. 
Unlike static graphs, TGs change dynamically and pose new challenges for graph learning. 
A central task in this domain is dynamic link prediction, which predicts whether two nodes will be connected at a specific time~\cite{poursafaei2022towards}. 
As a result,  efforts have been made to develop temporal graph neural networks (TGNNs) for temporal link prediction~\cite{cai2024survey}. 
These models show a strong capacity for capturing dynamic patterns in TGs, achieving promising performance.

Despite their success, TGNNs still exhibit inherent limitations that may impede their ability to capture the underlying structures of temporal graphs and constrain their effectiveness in dynamic link prediction. 
One key drawback is that most existing TGNNs focus predominantly on pairwise interactions, overlooking higher-order interactions and structures. 
Nodes in real-world evolving networks often interact within groups rather than solely through one-to-one connections~\cite{petri2018simplicial, battiston2021physics}. For example, a communication event may involve a team of individuals, and a transaction could include multiple stakeholders. 
It is evident that such group interactions -- formally referred to as higher-order structures -- are integral to the dynamic behavior and link formation of temporal graphs~\cite{benson2018simplicial}. 
However, how to effectively capture these higher-order structures in TGNNs remains underexplored.



Another critical challenge in current TGNN research is the limited expressive power~\cite{xu2018powerful} constrained by efficiency bottlenecks.
Existing studies have shown that a TGNN’s expressiveness is crucial for dynamic link prediction~\cite{souza2022provably}. 
One straightforward way to improve expressiveness is to expand the model’s receptive field to include multi-hop neighborhood information~\cite{xu2020inductive, wang2022inductiverepresentationlearningtemporal}. 
However, edges can recur multiple times between the same nodes in temporal graphs~\cite{poursafaei2022towards}, therefore, this expansion often leads to an exponential increase in training edges~\cite{luo2022neighborhood}, incurring high memory and computational costs.
There are methods~\cite{rossi2020temporal} that introduce a memory module to keep track of evolving features of each node, correspondingly improving the expressive power~\cite{souza2022provably}.
However, these modules require substantial resources in large-scale networks, and multi-hop aggregation remains necessary to mitigate memory staleness~\cite{rossi2020temporal} before computing output embeddings.
Consequently, overcoming these efficiency constraints while enhancing the model’s expressiveness power remains a significant challenge.


Based on these challenges, a natural question arises: {\it Can we design a TGNN that is both expressive and scalable?} To address this, we propose a novel \textit{Higher-order structure Temporal Graph neural Network}~(HTGN). 
To effectively leverage higher-order information, we propose an algorithm that identifies higher-order structures and constructs hyperedges to represent them.
Next, HTGN integrates temporal graph learning with hypergraph learning, aggregating multiple edge features into a single hyperedge feature, thereby reducing memory costs.
As a result, the hyperedge representations enable the HTGN to incorporate multiple hops of neighborhood information, enhancing the model's expressiveness.
Our main contributions can be summarized as follows:
\vspace{-2mm}
\begin{itemize}
    \item  We highlight the importance of higher-order structures in temporal graph learning.
To effectively capture and represent these structures, we formulate the problem of hyperedge construction in temporal graphs and propose an algorithm with a theoretical proof of its validity.

    \item We propose a novel \textit{Higher-order structure Temporal Graph neural Network}~(HTGN) that integrates temporal graph learning with hypergraph learning to better capture higher-order structures for temporal link prediction. Moreover, we theoretically demonstrate that HTGN exhibits greater expressiveness compared to baseline methods.
    \item We conduct extensive experiments on temporal graphs of various sizes and domains. The results demonstrate that HTGN can outperform baseline methods while also reducing memory costs by up to 50\%.
\end{itemize}
\vspace{-3mm}



\section{Related Works}

\textbf{Temporal Graph Neural Networks.}
Temporal graph learning has gained significant attention in recent years, with various kinds of temporal graph neural networks emerging to model dynamic relationships effectively. These methods can be roughly classified into two types: memory-based and subgraph-based. \textit{Memory-based methods} utilize a memory module to record the evolving dynamics. \citet{kumar2019predicting} introduced JODIE, a coupled recurrent neural network designed to learn the embedding trajectories of users and items. Similarly, DyRep~\cite{trivedi2019dyrep} focuses on generating low-dimensional node embeddings to encode communication patterns and associations. \citet{rossi2020temporal} proposed TGN, a temporal graph neural network that incorporates a memory module to maintain and update node representations. Recently, \citet{zhang2024efficient} integrated the Neural Common Neighbor with TGN, further improving the performance.
\textit{Subgraph-based methods} focus on utilizing temporal neighborhood structures in temporal graph learning. CAWN~\cite{wang2022inductiverepresentationlearningtemporal} leverages random anonymous walks to model the evolving neighborhood dynamics, while TCL~\cite{wang2021tcl} constructs a temporal dependency interaction graph to capture sequential cascades of interactions. TGAT~\cite{xu2020inductive} introduces a temporal graph attention mechanism to process time-aware neighborhood features. Additionally, NAT~\cite{luo2022neighborhood} employs a multi-hop neighboring node dictionary to extract joint neighborhood features. GraphMixer~\cite{cong2023do} and DyGFormer~\cite{yu2023towards} generate node features by aggregating one-hop neighborhood embeddings and input them to a transformer-based decoder~\cite{vaswani2017attention} to produce final predictions.

\textbf{Hypergraph Construction.} 
The problem of hypergraph (re)construction aims at identifying higher-order structures within graphs. Once these structures are identified, they are represented as hyperedges, allowing for the reconstruction of the entire network as a hypergraph.
Early studies~\cite{petri2013networks,petri2013topological,patania2017topological} have approached this problem by treating it as a maximal clique discovery task. However, \citet{young2021hypergraph} argue that not all cliques in a network are meaningful in real-world contexts. If all maximal cliques are indiscriminately considered higher-order structures, a substantial number of erroneous hyperedges will be introduced. To address this issue, they propose a Bayesian method to distinguish meaningful hyperedges from other maximal cliques. Furthermore, \citet{wang2024graphs} present a deep learning-based approach to identify hyperedges when a set of positive examples is available.

\section{Preliminaries}

In this section, we provide the basic notations used throughout the paper and define the problem formulation. Following existing works~\cite{rossi2020temporal, poursafaei2022towards, huang2023temporal}, we view the temporal graphs as a stream of events, where each event is an interaction between two nodes with a timestamp.

\begin{definition}[\textbf{Temporal Graphs}] 
\textit{There is a set of nodes $\mathcal{V}=\{n_1, n_2, ...\}$. A temporal graph $\mathcal{G}$ on $\mathcal{V}$ can be represented as a stream of events of node interactions $\{(u_1, v_1, t_1), (u_2, v_2, t_2), \cdots, (u_n, v_n, t_n)\}$, $t_1 \le t_2 \le \cdots \le t_n$, where $u_i, v_i \in \mathcal{V}$ are nodes that interact and $t_i$ denotes the timestamp of the $i$th link. We denote the set of dynamic links as $\mathcal{E}$.}
 \end{definition}
\vspace{-1mm}

A vital task for temporal graph learning is to predict how the graph evolves over time, which leads to the dynamic link prediction problem.

\begin{definition} [\textbf{Dynamic Link Prediction}] \textit{Given all historical events before time $t^*$, i.e., $\{(u, v, t)\in\mathcal{E} \: | \: t<t^*\}$, the dynamic link prediction problem is to predict whether there will be a link between two specific nodes $u^*$ and $v^*$ at time $t^*$.}
\end{definition}

To capture the higher-order structures of graphs, we propose to leverage hypergraphs and hyperedges in this work.

\begin{definition}
    [\textbf{Hypergraph}] 
    \textit{A hypergraph can be represented by a tuple $(\mathcal{V}, \mathcal{H})$. $\mathcal{V}=\{n_1, n_2, ...\}$ is the set of nodes. And the set of hyperedges is $\mathcal{H}=\{E_i \: | \: E_i\subseteq \mathcal{V}; \: 1\leq i \leq m\}$, where $E_i$ is hyperedge and $m$ is the number of hyperedges.}
\end{definition}

Furthermore, in homogeneous graphs, the construction of hyperedges usually relies on the enumeration of maximal cliques~\cite{wang2024graphs}. 

\begin{definition}[\textbf{Maximal Cliques}]\label{def:clique}
\textit{
    A clique is a fully connected subgraph. A maximal clique is a clique that is not the subgraph of any other clique.
    }
\end{definition}

In the next section, we will introduce our method for modeling higher-order structures in temporal graph neural networks ({\bf HTGN}). 
The method consists of two major components: (1) HTGN identifies the higher-order structures within the temporal graphs and represents them with hyperedges; and (2) HTGN learns and predicts dynamic links based on the higher-order structural information of temporal graphs. 



\section{Hyperedge Construction from Temporal Graphs}\label{subsec:hyperedge construction}

Many previous works have highlighted the importance of higher-order structures for analyzing networks~\cite{benson2016higher,benson2018simplicial, zitnik2018modeling}. Usually, the higher-order structures are represented by hyperedges, indicating that interactions occur among more than two nodes. Hence, one key challenge for our method is how to identify and construct the hyperedges \textit{efficiently} and \textit{accurately}.

The following discusses our proposed hyperedge construction method for the two different types of temporal graphs, i.e., homogeneous temporal graphs and heterogeneous temporal graphs. We formally define the two types of networks as following: 
\begin{definition}[Homogeneous temporal graphs]
\textit{ In homogeneous temporal graphs, all nodes belong to the same type, and interactions can occur between any two nodes. }
\end{definition}

\begin{definition}[Heterogeneous temporal graphs] 
\textit{In heterogeneous temporal graphs, nodes can belong to different types, with interactions occurring exclusively between nodes of two distinct types. In this study, we focus on heterogeneous graphs with exactly two node types. \textbf{Hence, all the heterogeneous graphs in this work are bipartite.}}
    
\end{definition}



\subsection{Homogeneous Graphs} 
For homogeneous graphs, a hyperedge contains an arbitrary number of nodes that are interconnected. 
Therefore, it is natural to construct hyperedges on the basis of maximal cliques in Def \ref{def:clique}.
Nonetheless, as pointed out by \citet{wang2024graphs}, not every maximal clique in graphs represents a real higher-order structure that can be viewed as a hyperedge~(an example to illustrate this is provided in Appendix~\ref{app:case}).
In general cases, \textit{if no additional information is provided}, directly transforming every maximal clique to a hyperedge will introduce a large amount of errors~\cite{wang2024graphs}.


Fortunately, the temporal behaviors of links in temporal graphs can serve as useful information and help to handle this challenge. It is evident that \textit{nodes within a hyperedge are more likely to interact collectively within a short time slot}~\cite{benson2018simplicial}.
This indicates that if we accumulate the dynamic links in a time interval to form a snapshot, the maximal cliques within the snapshot are likely to be real hyperedges.
Next, we will formally illustrate this intuition. To achieve this goal, we first extend the hypergraph stochastic block model~\cite{pister2024stochastic} to temporal graphs and propose the hypergraph temporal stochastic block model~(HT-SBM).

\begin{definition}[Hypergraph temporal stochastic block model]
Let \( V = \{1, 2, \dots, n\} \) be the set of nodes, divided into \( K \) communities \( \{C_1, C_2, \dots, C_K\} \), where \( \bigcup_{a=1}^K C_a = V \). 
Let \( \Lambda_t \in \mathbb{R}^{K \times K} \) denote the community connection probability matrix at time \( t \), where \( \Lambda_t[a, b] \) represents the connection probability between communities \( C_a \) and \( C_b \). The time evolution of connection probabilities is given by:
\[
\Lambda_t[a, b] = \Lambda_0[a, b] \cdot \pi(t),
\]
where \( \pi(t) \) is the probability density function with respect to $t$ function, which means evolving with time. 
A hyperedge \( E \subseteq V \) of size \( k \) is said to \emph{exist} at time \( t \) if every pair of nodes \( i, j \in E \) is connected:
\[
\Pr(E \text{ exists at time } t) = \prod_{i, j \in E, i < j} \Pr((i, j) \in A_t),
\]
where \( A_t \) is the adjacency matrix of the network snapshot at time \( t \).
\end{definition}


Using HT-SBM, we now theoretically establish the relationship between the accuracy of the maximal clique algorithm in constructing hyperedges on a snapshot of a temporal graph and the duration of the snapshot. 

\begin{theorem}\label{thm:main}
Given an HT-SBM model of $n$ nodes,
let \( P \) be an exponential distribution with rate parameter \( \lambda \), and let \( Q \) be the probability distribution of $\Lambda_t$. 
Consider a snapshot of the HT-SBM model within duration $t$. Denote by \( Acc(\mathcal{H}) \) the accuracy of the maximal clique algorithm for constructing hyperedges within the snapshot. Then, the expected accuracy satisfies:
\begin{equation}
\begin{aligned}
    \min_{\mathcal{H}} \mathbb{E}_{\Lambda_t \sim Q} \big[ Acc(\mathcal{H}) \big] 
    \leq  D_{\text{KL}}(Q \| P) -\\  \textit{Const} \cdot \frac{n(n-1)}{2} \cdot (1 - e^{-\lambda t})^{2K}
\end{aligned}
\end{equation}
where \textit{Const} is a positive constant scalar. and $D_{\text{KL}}$ is the KL divergence.
\end{theorem}

The detailed proof can be found in Appendix~\ref{app:proof1}.
Theorem~\ref{thm:main} suggests that the longer the snapshot duration $t$ is, the lower the hyperedge construction accuracy would be. This fits with the empirical observations found by~\citet{benson2018simplicial}.
Therefore, to accurately construct hyperedges, the maximal clique algorithm {\it should be applied to a snapshot of limited duration}. However, the duration must not be too short, as hyperedges require sufficient time to form.
To address this, we propose an algorithm for hyperedge construction, which applies a hyperparameter to control the size of each snapshot.
The algorithm is described in Algorithm~\ref{alg:example}. 

\begin{algorithm}[tb]
   \caption{Hyperedge Construction}
   \label{alg:example}
\begin{algorithmic}[1]

   \STATE {\bfseries Input:} A event-stream source will generate a link $(u,v)$ at every time step $t$.
   \STATE {\bfseries Initialize:} A set $\mathcal{H}$ to record all the hyperedges constructed. An empty graph snapshot $\mathcal{S}$.
   \REPEAT
   \STATE Receive a new link $(u, v)$
   \IF{$(u, v)$ is already in a hyperedge in $S$} 
   \STATE \textit{Continue}
   \ELSE 
   \STATE $\mathcal{H}$.add($(\{u,v\}$)
   \ENDIF
   \STATE $\mathcal{S}$.add\_edge($u, v$)
   \IF{$\mathcal{S}$ has more than $b$ edges}
   \STATE Find the set $\mathcal{C}$ of all the maximal cliques with more than 2 nodes in $S$
   \FOR{$c$ in $\mathcal{C}$}
   \STATE $\mathcal{H}$.add($c$)
   \FOR{$E_i  \in \{E\subset c |  E \in \mathcal{H}\}$}
   \STATE $\mathcal{H}$.delete($E_i$)
   \ENDFOR
   \ENDFOR
   \STATE Reset $\mathcal{S}$ to be empty.
   \ENDIF
   \UNTIL{No new link is generated}
\end{algorithmic}
\end{algorithm}


\begin{figure}[ht]
\vskip 0.2in
\begin{center}
\centerline{\includegraphics[width=0.7\columnwidth]{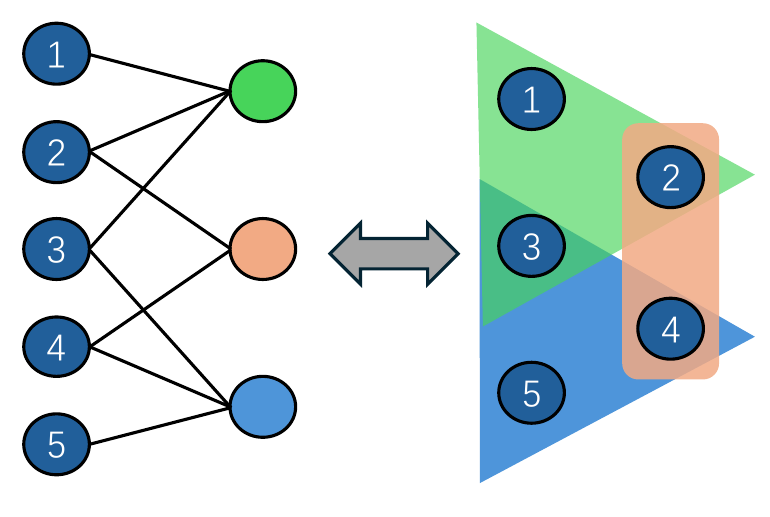}}
\caption{The transformation between heterogeneous bipartite graphs and hypergraphs.}
\label{fig:bipartite}
\end{center}
\vskip -0.2in
\end{figure}

\subsection{Heterogeneous~(Bipartite) Graphs}  
For heterogeneous bipartite graphs, we adopt the approach from~\cite{berge1984hypergraphs} to construct hypergraphs. 
Unlike homogeneous graphs for which we need to recognize the latent hyperedges, bipartite graphs are treated as \textit{explicit representations} of higher-order structures~\cite{kirkley2024inference}. 
Specifically, nodes from one partition connected to the same node in the other partition form a hyperedge as shown in Figure~\ref{fig:bipartite}.
We extend this idea to the bipartite temporal graphs.
To account for the dynamic structural changes in temporal graphs, we introduce a time threshold $t'$ for hyperedge construction. For example, in a user-item e-commerce network, users who access the same item within a recent time interval $t'$ are grouped into the same hyperedge. 
Through this method, one partition of nodes is transformed into hyperedges that connect the other partition. 
The detailed algorithm is in Appendix~\ref{app:bipartite}.



\begin{figure*}[ht]
\vskip 0.2in
\begin{center}
\centerline{\includegraphics[width=2\columnwidth]{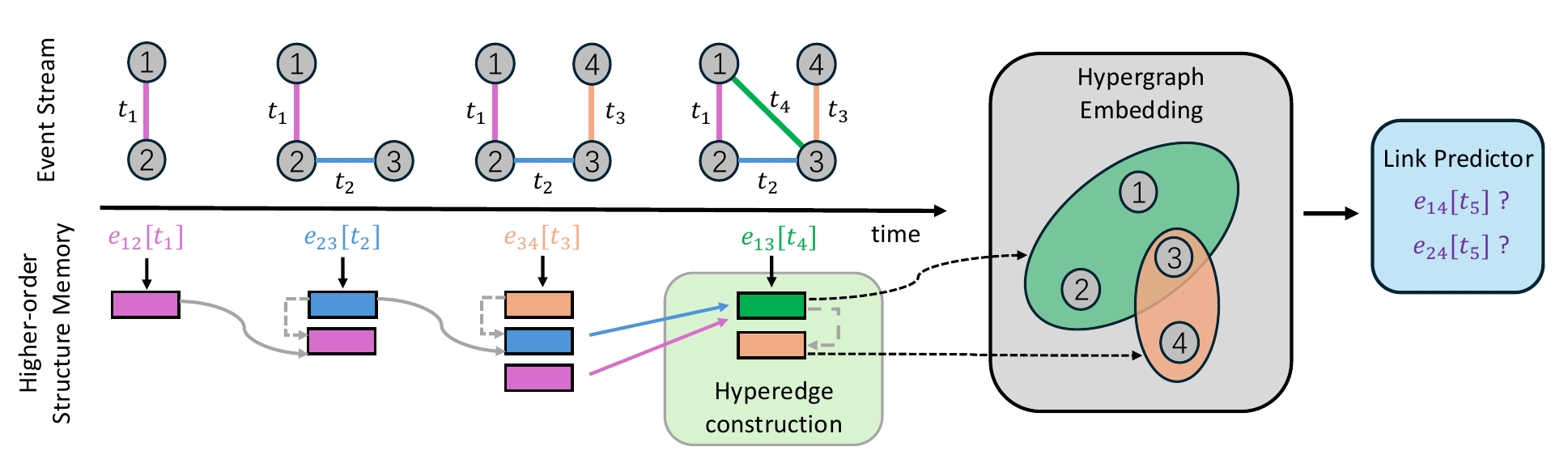}}

\caption{The pipeline of HTGN. \textit{Left:} the higher-order structure memory module. It will be dynamically updated as the stream of interaction events is input into the model~(a homogeneous temporal graph in this example). \textit{Middle:} the hypergraph embedding module. It will compute the node embeddings with the hyperedge features. \textit{Right:} the link predictor. It will predict whether two nodes will form a link at the given time according to their node embeddings.}
\label{fig:pipeline}
\end{center}
\vskip -0.2in
\end{figure*}

\section{Higher-order structure Temporal Graph Network~(HTGN)}

In this subsection, we present the design of our proposed higher-order structural temporal graph network.  An overview of HTGN is shown in Figure~\ref{fig:pipeline}. It consists of two major components: {\bf (1)} a higher-order structure memory module and {\bf (2)} a hypergraph embedding module. We will introduce these two components in detail respectively.

\subsection{Higher-order Structure Memory Module} 
To track the evolving features of each node, the traditional memory module allocates a specific memory unit to each node in the temporal graph~\cite{rossi2020temporal, zhang2024efficient}. However, this approach becomes impractical when dealing with large-scale temporal graphs due to excessive memory consumption. Moreover, relying solely on node features may fail to capture essential structural information for link prediction. 
To address these challenges, we propose a dynamic memory allocation strategy that assigns memory units to higher-order structures rather than statically allocating a fixed number of units per node. 
This approach effectively reduces memory costs while enabling HTGN to track crucial higher-order structures within the network.

In the following, we describe the initialization and update process of the HTGN memory module as the graph structure evolves.
For homogeneous graphs, HTGN seamlessly integrates Algorithm~\ref{alg:example} into its workflow, dynamically updating the hyperedges as new links are introduced chronologically.
For heterogeneous graphs, while the number of hyperedges remains fixed at the beginning, the nodes within each hyperedge dynamically change over time.

\textbf{Memory Initialization.} 
Before the training starts, HTGN initializes $z$ memory units, where each memory unit keeps track of a hyperedge feature. We denote the set of memory features as $\mathcal{M}$, then we get
\begin{equation}
    \mathcal{M} = \{\mathbf{m_1}, \mathbf{m_2}, ... , \mathbf{m_z} \}
\end{equation}

For homogeneous graphs, $z$ is smaller than the total number of nodes in the graph. We use the set $\mathcal{H}$ to record the hyperedges constructed.
When there is an interaction event with feature $\mathbf{e}_{ij}(t)$ between node $n_i$ and $n_j$ at time $t$, if at least one of node $n_i$ and node $n_j$ does not
belong to any hyperedges in $\mathcal{H}$, we will create a new hyperedge $E^*=\{n_i, n_j\}$ and add $E$ to $\mathcal{H}$~(Line 5 to 9 in Algorithm~\ref{alg:example}). We assign an empty memory unit $\mathbf{m_i}$ to store the link feature. For simplicity, we suppose that the memory features can be indexed by hyperedges, such that $\mathbf{m}[E^*] = \mathbf{e}_{ij}(t)$.

For heterogeneous graphs, $z$ equals to the number of hyperedges fixed in Section~\ref{subsec:hyperedge construction}. Hence, we have $\mathbf{m}_i = \mathbf{m}[E_i]$, and each memory unit would keep track of the evolving features of a hyperedge, respectively.

\textbf{Higher-order Structure Update.} As the structure of temporal graphs evolving over time, the model would also keep updating the higher-order structure during the training and inferring.
For simplicity of notation, we define two maps at the beginning of our discussion.

\begin{definition}
    The map $H[n]$ will return the set of all hyperedges that a node $n$ belongs to, i.e., $H[n] = \{E| n \in E\}$.
\end{definition}

\begin{definition}
    The map $t[n, E]$ will return the most recent time that node $n$ has an interaction in the hyperedge $E$. And we note that $t[E] = \max(t[\cdot, E])$.
\end{definition}

Next, we will discuss the higher-order structure update for homogeneous and heterogeneous graphs, respectively.

For homogeneous graphs, we follow Algorithm~\ref{alg:example}. 
Every time that a batch of $b$ interactions occur, we would enumerate the maximal cliques $\mathcal{C}$ with the snapshot formed by the $b$ links.
Suppose that the higher-order structure recorded by the model is $\mathcal{S}$.
For a new maximal clique $c$ containing more than 2 nodes recognized in the snapshot, we will set it as a new hyperedge. 
The hyperedges to be aggregated into $c$ can be noted by $\mathcal{D}_c \coloneqq \{E_i\subset c \: | \: E_i\in \mathcal{H}\}$~(Line 11 to 19 in Algorithm~\ref{alg:example}). 
Then, the model aggregates the features of hyperedges in $\mathcal{D}_c$ into a single memory feature $\mathbf{m}_c$.
To mitigate the information loss of aggregation, we follow the previous works~\cite{souza2022provably} to employ an injective function for the aggregation operation. Specifically, we multiply an exponential term to each memory feature, which assigns smaller weights to older features. The new memory feature is calculated by:
\begin{equation}
    \mathbf{m}[c] = \sum_{E \in \mathcal{D}_c} \text{MLP}(\mathbf{m}[E] || )\alpha^{-\beta (t^*-t[E])}
\end{equation}

where $t^*$ is the current time, and $\alpha$, $\beta$ are scaler hyperparameters. After the calculation, we delete all hyperedges in $\mathcal{D}_c$ from $\mathcal{H}$ and set their memory units to be empty.

For heterogeneous graphs, since the number of hyperedges is fixed, we only change the values of $H(n)$. Each node will only be in a hyperedge if there is an interaction between them during a recent time interval $t'$. If a hyperedge $E$ fulfills $t^*-t[n,E] > t'$, it will be removed from the $H[n]$.

\textbf{Memory Feature Update.}
Besides the memory initialization and structure update, the hyperedge feature will also be updated whenever there is an interaction.
Suppose the interaction happens between node $u$ and node $v$ at time $t$ with link feature $e_{uv}(t)$, the message will be passed between the hyperedges which $u$ and $v$ are in.
The model first randomly selects hyperedges $E_u$ and $E_v$ from $H[u]$ and $H[v]$ respectively. Then their features $\mathbf{m}[E_u]$ and $\mathbf{m}[E_v]$ will be updated. Specifically, the updated features $\mathbf{m}'[E_u]$ and $\mathbf{m}'[E_v]$  will be calculated as:
\begin{align}
    &\mathbf{f}_u = \text{msg}_s(\mathbf{m}[E_u], \mathbf{m}[E_v], \Delta t, e_{uv}(t), |E_u|) \\
    &\mathbf{f}_v = \text{msg}_d(\mathbf{m}[E_u], \mathbf{m}[E_v], \Delta t, e_{uv}(t), |E_v|)\\
    &\mathbf{m}'[E_u] = \text{mem}(\mathbf{f}_u, \mathbf{m}[E_u])\\
    &\mathbf{m}'[E_v] = \text{mem}(\mathbf{f}_v, \mathbf{m}[E_v])
\end{align}
where $|\cdot|$ is the size of a hyperedge, $\text{msg}_s()$ and $\text{msg}_g()$ are message functions, $\text{mem}()$ is a memory update function, and $\Delta t$ is the time since last interaction between the two nodes. 
In our implemention, we use concatenation of all inputs for $\text{msg}_s$ and $\text{msg}_g$, and GRU~\cite{cho2014learning} for the $\text{mem}()$ function.


\subsection{Hypergraph Embedding} 
Since the memory module records the hyperedge features, the model needs to calculate the node embeddings before predicting future links.
To fully utilize the hyperedge features, we apply hypergraph convolution~\cite{bai2021hypergraph} to compute the node embeddings.
Before presenting the details, we first define the concept of k-hop temporal neighborhood.
\begin{definition}
    For a node $u$, its $k$-hop neighborhood at time $t$ is a set of nodes which has distance $k$ from $u$. We denote the set as $\mathcal{N}^t_k(u) \coloneqq \{v \: | \: SPD^t(u,v)=k \}$, where $SPD^t$ is the shortest path distance by the number of edges before time $t$. 
\end{definition}

For a node $u$, its initial embedding is calculated as: 
\begin{align}
    \mathbf{h}^{(1)}_u(t)&=\mathrm{ReLu}(\sum_{E  \in H[u] } \mathbf{W}_1^{(1)}(\mathbf{m}[E] \,  \|\, \boldsymbol{\phi}(\delta t)))
\end{align}
where $\delta t = t - t[u, E]$ is the change in time and $\boldsymbol{\phi}(\cdot)$ is a time encoding function which is detailed in Appendix~\ref{app:time encoding}. $\mathbf{W}_1^{(1)}$ is a learnable weight matrix.
Then the $l$-th layer embedding of node $u$ is calculated as:
\begin{align}
    &\mathbf{h}^{(l)}_u = \mathbf{W}_2^{(l)}(\mathbf{h}^{(l-1)}_u(t) \, \| \, \tilde{\mathbf{h}}^{(l)}_u(t))\\
    &\tilde{\mathbf{h}}^{(l)}_u = \mathrm{ReLu} (\sum_{j  \in \mathcal{N}^t_k(u) } \mathbf{W}_1^{(l)}(\mathbf{h}^{(l-1)}_j \,  \|\, \boldsymbol{\phi}(\Delta t)))
\end{align}
Here as well, $\boldsymbol{\phi}(\cdot)$ is a time encoding and $\mathbf{W}_1^{(l)}$ and $\mathbf{W}_2^{(1)}$ are both learnable weight matrices.

\textbf{Link prediction.}
To predict whether a link will form at given time $t$, we train a link predictor with the node embeddings as input. 
For homogeneous graphs, given the two target node embeddings $\textbf{h}^t_u$ and
$\textbf{h}^t_v$ at time $t$, the probability $p((u,v)|t)$ for them to form a link is calculated as:
\begin{equation}
    p((u,v)|t) = \text{Sigmoid}(\text{MLP}_{link}(\textbf{h}^t_u\otimes \textbf{h}^t_v))
\end{equation}
where $\otimes$ stands for elementwise multiplication. For heterogeneous graphs, the node $v$ is viewed as a hyperedge $E_v$. Hence, $p((u,v)|t)$ for them to form a link is calculated as:
\begin{equation}
    p((u,v)|t) = \text{Sigmoid}(\text{MLP}_{link}(\textbf{h}^t_u\otimes \textbf{m}^t[E_v]))
\end{equation}


\textbf{Expressive power of HTGN.} From the theoretical perspective, we analyze the expressive power of HTGN, which is widely regarded as a key factor in determining the model's capability for link prediction~\cite{souza2022provably}. Due to the ability to represent higher-order structures effectively, HTGN is more powerful than traditional pairwise message-passing temporal graph neural networks~(e.g., TGN, TGAT, GraphMixer, and so on). 
Formally, we develop the theorem below:
\begin{theorem}\label{expressive}
    HTGN is strictly more expressive than pairwise message-passing TGNN.
\end{theorem}
For the detailed proof, please refer to Appendix~\ref{app:proof2}. 

\section{Experiment}

In this section, we aim to answer the following research questions with experiments: 
\textbf{RQ1:} How does HTGN perform compared to the baselines?
\textbf{RQ2:} Is HTGN efficient compared to the baselines? and
\textbf{RQ3:} How do the key modules of HTGN affect its performance?


\subsection{Experimental Settings}

\setlength{\textfloatsep}{3pt}

\begin{table*}[!t]
\centering
\caption{Performances on \textbf{homogeneous} graphs.
The metric in the table is MRR, where higher values indicate better performance. 
All numerical results are averages across 3 random runs. Numbers in \textbf{\textcolor{red}{red}} represent the best results and \textbf{\textcolor{blue}{blue}} the second best. OOT indicates the out-of-time problem.}

\label{table:homo}

\vspace{0.1cm}
\resizebox{\textwidth}{!}
{
\begin{tabular}{lllcccccccccccc}
\toprule[2pt]
\multicolumn{3}{c}{\multirow{2}{*}{Method}} & \multicolumn{2}{c}{UCI} & \multicolumn{2}{c}{Enorn} & \multicolumn{2}{c}{Social Evo.}   & \multicolumn{2}{c}{tgbl-coin} & \multicolumn{2}{c}{tgbl-flight} & \multicolumn{1}{c}{\multirow{2}{*}{Rank}}\\

\cmidrule(r){4-5} \cmidrule(r){6-7} \cmidrule(r){8-9} \cmidrule(r){10-11} \cmidrule(r){12-13} 
& & & \multicolumn{1}{c}{Val} & \multicolumn{1}{c}{Test} & \multicolumn{1}{c}{Val} & \multicolumn{1}{c}{Test} &\multicolumn{1}{c}{Val} & \multicolumn{1}{c}{Test} &\multicolumn{1}{c}{Val} & \multicolumn{1}{c}{Test} &\multicolumn{1}{c}{Val} & \multicolumn{1}{c}{Test} \\
\hline

\multicolumn{3}{l}{TGN} & 29.4 & 29.9 & 50.7 & 51.1 & 34.2 &35.6 & 72.2 & 73.1   & 70.1 & 70.4& 4.50\\
        \multicolumn{3}{l}{TNCN} & 34.2 & 35.3 & 53.7 & 54.4 & 43.2 & 42.3 & \textbf{\textcolor{blue}{72.5}} &\textbf{\textcolor{blue}{73.9}} & \textbf{\textcolor{blue}{78.3}} & \textbf{\textcolor{blue}{79.8}}  & \textbf{\textcolor{blue}{2.80}}\\ 
        \multicolumn{3}{l}{DyGformer} & 34.7 & 35.9 & \textbf{\textcolor{red}{58.4}} & \textbf{\textcolor{red}{59.2}} & 33.7 & 34.5 & OOT & OOT & OOT & OOT &4.00\\ 
        \multicolumn{3}{l}{CAWN} & \textbf{\textcolor{blue}{35.9}} & \textbf{\textcolor{blue}{37.3}} & 41.3 & 42.4 & 38.9 & 39.2 & OOT & OOT & OOT & OOT&4.33\\ 
        \multicolumn{3}{l}{NAT} & 33.1 & 33.7 & 49.1 & 47.3 & \textbf{\textcolor{red}{49.5}} & \textbf{\textcolor{red}{48.1}} & OOT & OOT & OOT & OOT & 3.67 \\ 
        \multicolumn{3}{l}{GraphMixer} & 17.6 & 20.4 & 21.9 & 23.7 & 35.5 & 35.3 & OOT & OOT & OOT & OOT &7.50\\ 
        \multicolumn{3}{l}{TGAT} & 20.4 & 23.5 & 36.3 & 35.8 & 40.4 & 41.6 & OOT & OOT & OOT & OOT &6.00 \\ 
        \multicolumn{3}{l}{TCL} & 13.6 & 12.8  & 14.5 & 14.8 & 24.6 & 24.8 & OOT & OOT & OOT & OOT &9.00\\ 
        \multicolumn{3}{l}{DyRep} & 7.2 & 5.0 & 11.2 & 12.7 & 15.6 & 19.4 & 45.2 & 44.8 &55.6 &54.5&7.60 \\ 
         \hline
        \multicolumn{3}{l}{HTGN} & \textbf{\textcolor{red}{38.1}} & \textbf{\textcolor{red}{37.8}} & \textbf{\textcolor{blue}{56.8}} & \textbf{\textcolor{blue}{57.3}} & \textbf{\textcolor{blue}{44.8}} & \textbf{\textcolor{blue}{44.3}} & \textbf{\textcolor{red}{75.6}} & \textbf{\textcolor{red}{75.8}} & \textbf{\textcolor{red}{79.4}} & \textbf{\textcolor{red}{80.5}} & \textbf{\textcolor{red}{1.40}} \\
\bottomrule[2pt]
\end{tabular}
}
\end{table*}

\setlength{\textfloatsep}{3pt}

\begin{table*}[!t]
\centering
\caption{Performances on \textbf{heterogeneous} graphs.
The metric in the table is MRR, where higher values indicate better performance. 
All numerical results are averages across 3 random runs. Numbers in \textbf{\textcolor{red}{red}} represent the best results and \textbf{\textcolor{blue}{blue}} the second best. }

\label{table:hetero}

\vspace{0.1cm}
\resizebox{\textwidth}{!}
{
\begin{tabular}{lllcccccccccccc}
\toprule[2pt]
\multicolumn{3}{c}{\multirow{2}{*}{Method}} & \multicolumn{2}{c}{tgbl-wiki} & \multicolumn{2}{c}{Reddit} & \multicolumn{2}{c}{LastFM}   & \multicolumn{2}{c}{MOOC} & \multicolumn{2}{c}{tgbl-review} & \multicolumn{1}{c}{\multirow{2}{*}{Rank}}\\

\cmidrule(r){4-5} \cmidrule(r){6-7} \cmidrule(r){8-9} \cmidrule(r){10-11} \cmidrule(r){12-13} 
& & & \multicolumn{1}{c}{Val} & \multicolumn{1}{c}{Test} & \multicolumn{1}{c}{Val} & \multicolumn{1}{c}{Test} &\multicolumn{1}{c}{Val} & \multicolumn{1}{c}{Test} &\multicolumn{1}{c}{Val} & \multicolumn{1}{c}{Test} &\multicolumn{1}{c}{Val} & \multicolumn{1}{c}{Test} \\
\hline

\multicolumn{3}{l}{TGN} & 61.5 & 61.6 & 51.4 & 57.0 & 48.3 & 49.7 & 35.9 & 37.8  & 51.2 & \textbf{\textcolor{blue}{52.8}}& 4.0 \\
        \multicolumn{3}{l}{TNCN} & 71.1 & 70.9 & 49.3 & 47.2 & 43.8 & 42.6 & 40.6 & 41.1 & 32.8 &38.1 & 4.5\\ 
        \multicolumn{3}{l}{DyGformer} & \textbf{\textcolor{blue}{78.2}} & \textbf{\textcolor{blue}{77.4}} & 27.3 & 26.5 & \textbf{\textcolor{blue}{54.7}} & \textbf{\textcolor{blue}{51.6}} & 28.9 & 31.4 & 34.1 & 35.4 & 5.1 \\ 
        \multicolumn{3}{l}{CAWN} & 70.3 & 71.1 & 29.2 & 28.1 & 37.5 & 36.2 & 35.6 & 35.9 & 24.5 & 23.4 & 6.7\\ 
        \multicolumn{3}{l}{NAT} & 74.9 & 73.6 & \textbf{\textcolor{blue}{58.4}} & \textbf{\textcolor{blue}{59.2}} & 49.5 & 48.1 & 43.2 & 42.7 & 34.5 & 35.6 & \textbf{\textcolor{blue}{3.2}} \\ 
        \multicolumn{3}{l}{GraphMixer} & 17.6 & 20.4 & 21.9 & 23.7 & 35.5 & 35.3 & \textbf{\textcolor{blue}{43.8}} & \textbf{\textcolor{blue}{44.4}} & \textbf{\textcolor{red}{52.1}} & \textbf{\textcolor{red}{53.2}} & 6.0\\ 
        \multicolumn{3}{l}{TGAT} & 20.4 & 23.5 & 36.3 & 35.8 & 40.4 & 41.6 & 34.2 & 35.3 & 30.8 &31.2 & 6.4\\ 
        \multicolumn{3}{l}{TCL} & 19.8 & 20.7  & 32.4 & 33.2 & 38.6 & 37.9 & 22.7 &21.3 & 19.9 & 19.3 & 8.0\\ 
        \multicolumn{3}{l}{DyRep} & 7.2 & 5.0 & 11.2 & 12.7 & 15.6 & 19.4 & 18.3 & 17.2 &21.6 &22.0 & 9.8\\ 
         \hline

        \multicolumn{3}{l}{HTGN} & \textbf{\textcolor{red}{79.6}} & \textbf{\textcolor{red}{78.3}} & \textbf{\textcolor{red}{60.8}} & \textbf{\textcolor{red}{59.7}} & \textbf{\textcolor{red}{55.8}} & \textbf{\textcolor{red}{56.9}} & \textbf{\textcolor{red}{45.6}} & \textbf{\textcolor{red}{45.8}} & \textbf{\textcolor{blue}{51.5}} & 52.2 & \textbf{\textcolor{red}{1.3}} \\
\bottomrule[2pt]
\end{tabular}
}
\end{table*}

\begin{figure*}[ht]
\vskip 0.2in
\begin{center}
\centerline{\includegraphics[width=2\columnwidth]{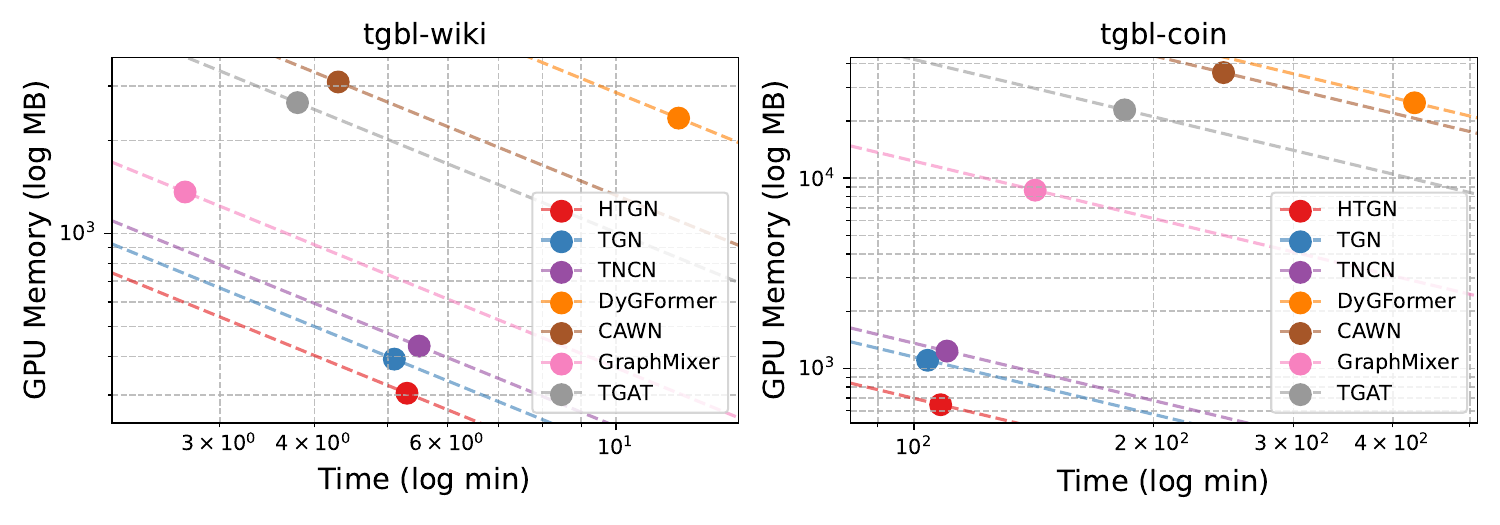}}
\vskip -0.2in
\caption{The efficiency of HTGN compared to the baselines. The x-axis indicates the running time of a single epoch, the y-axis indicates the model's GPU memory costs. Lower x-axis and y-axis metrics indicate better efficiency. HTGN exhibits considerable efficiency advantages against the baselines, especially in GPU memory costs.}
\label{fig:efficiency}
\end{center}
\vskip -0.2in
\end{figure*}

\textbf{Datasets and Evaluation.} 
We evaluate our models on ten real-world datasets, which are from the work of \citet{poursafaei2022towards} and \textit{Temporal Graph Benchmark}~\cite{huang2023temporal}.
These datasets include both homogeneous and heterogeneous networks from diverse domains, such as communication, transportation, and social networks. 
For further details of the datasets, please refer to Appendix~\ref{app:dataset}. 
We follow previous works~\cite{patania2017topological} and apply both random sampling and historical sampling to collect negative samples.
Following~\cite{poursafaei2022towards, huang2023temporal}, we adopt the Mean Reciprocal Rank~(MRR) as the metric for evaluation.

\textbf{Baselines.} 
We consider two types of representative baselines. For memory-based method, we evaluate DyRep~\cite{trivedi2019dyrep}, TGN~\cite{rossi2020temporal} and TNCN~\cite{zhang2024efficient}. For graph-based methods, we evaluate TGAT~\cite{xu2020inductive}, TCL~\cite{wang2021tcl}, CAWN~\cite{wang2022inductiverepresentationlearningtemporal}, DyGformer~\cite{yu2023towards} and GraphMixer~\cite{cong2023do}.

\begin{figure}[ht]
\begin{center}
\centerline{\includegraphics[width=\columnwidth]{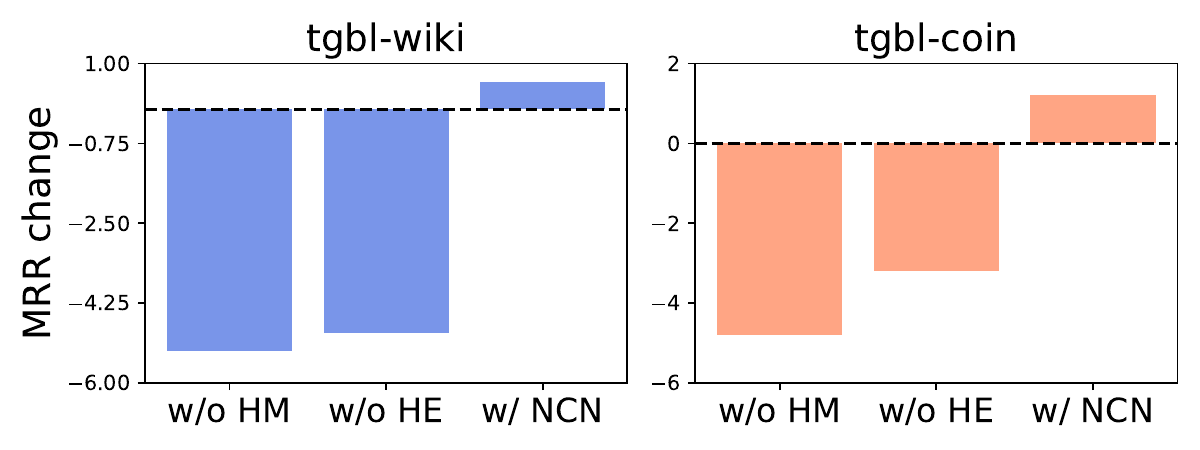}}
\vskip -0.1in
\caption{The change of MRR compared to original HTGN after removing/adding modules. HM indicates higher-order structure memory, HE indicates hypergraph embedding, and NCN indicates the common-neighbor information~\cite{wang2023neural}.}
\label{fig:ablation}
\end{center}
\vskip -0.1in
\end{figure}

\subsection{RQ1: Dynamic Link Prediction Performance}\label{subsec:performance}
\textbf{Setup.} 
We first compare the dynamic link prediction performance of HTGN against the baselines on both homogeneous and heterogeneous graphs. 
For a fair comparison, we keep the size of the node and edge embeddings the same for all the methods.
Specifically, we set the dimension of node embedding to be $100$ and the dimension of edge embedding to be $175$. 
Moreover, except GraphMixer~\cite{cong2023do} which fixes the number of hops equal to 1,  we set the number of neighborhood hops to 2 for all methods.

\textbf{Results.} 
The results are shown in Table~\ref{table:homo} and \ref{table:hetero}. From the tables, the following observations can be made:
{\bf Obs. (1)} 
From the results in Tables~\ref{table:homo} and \ref{table:hetero}, we can see that HTGN achieves superior performance compared to existing baselines. Specifically, it is either the 1st or 2nd best method on 9/10 datasets. Furthermore, it yields an {\bf average improvement of 5\% over the second-best method} on both heterogeneous and homogeneous datasets, underscoring its effectiveness.
{\bf Obs. (2)} The performance of baselines can significantly differ between homogeneous and heterogeneous datasets. For instance, TNCN~\cite{zhang2024efficient}'s ranks have a mean rank of 2.8 on the homogeneous datasets but only 4.5 on the heterogeneous ones. In contrast, HTGN consistently performs well on both dataset types, with a mean rank of 1.3 and 1.4, respectively.
{\bf Obs. (3)} Compared to TGN, whose vanilla memory module stores node features without leveraging higher-order structural information, HTGN demonstrates significant performance gains. This finding highlights the importance of higher-order structures in dynamic link prediction. 

\subsection{RQ2: Efficiency Comparison}
\textbf{Setup.} Beyond performance, we are also interested in each model’s efficiency, particularly GPU memory consumption and training time. To ensure fairness, we use the same settings as that in Section~\ref{subsec:performance} and record memory and time costs for a single training epoch. 
HTGN converges within a similar number of epochs as TNCN and TGN.

\textbf{Results.} 
The results are shown in Figure~\ref{fig:efficiency}. 
Compared to subgraph-based methods~\cite{xu2020inductive,yu2023towards}, memory-based approaches require less GPU memory and training time, consistent with the observations by \citet{rossi2020temporal} and \citet{zhang2024efficient}.
Furthermore, among memory-based methods (TGN and TNCN), HTGN reduces GPU memory usage by 30\%–50\%, showing that storing higher-order structural features in memory effectively lowers resource costs.
Lastly, despite performing online updates of higher-order structures, HTGN does not incur significant additional time overhead compared to TGN or TNCN, confirming the efficiency of our proposed hyperedge construction algorithm.

\subsection{RQ3: Ablation Study}\label{subsec:ablation}


We study the impact of the different modules in our framework. The result is shown in Figure~\ref{fig:ablation}). Specifically, we replaced the high-order memory module with vanilla node memory and the hypergraph convolution with graph convolution. Both substitutions led to substantial performance drops, indicating that these modules each contribute to performance and should be integrated to fully leverage the expressive power of higher-order structures. We also examined whether incorporating common neighbor information could further improve HTGN. This addition yielded a small performance gain, suggesting that common neighbor information may capture structural cues not fully captured by HTGN.

\section{Conclusion}
In this work, we address the challenge of designing temporal graph neural networks that are both efficient and expressive. We propose HTGN, which significantly reduces memory costs while enhancing link prediction performance by incorporating higher-order structural information.
A potential future direction for this research is the development of a temporal graph foundation model~\cite{maoposition,shirzadkhani2024towards} capable of generalizing across multiple datasets by leveraging commonly shared higher-order structural information.


\bibliography{main}
\bibliographystyle{icml2025}

\newpage
\appendix
\onecolumn

\section{Proof of Theorem~\ref{thm:main}}\label{app:proof1}

\begin{definition}
Let \( V = \{1, 2, \dots, n\} \) be the set of nodes, divided into \( K \) communities \( \{C_1, C_2, \dots, C_K\} \), where \( \bigcup_{a=1}^K C_a = V \). 
Let \( \Lambda_t \in \mathbb{R}^{K \times K} \) denote the community connection probability matrix at time \( t \), where \( \Lambda_t[a, b] \) represents the connection probability between communities \( C_a \) and \( C_b \). The time evolution of connection probabilities is given by:
\[
\Lambda_t[a, b] = \Lambda_0[a, b] \cdot \pi(t),
\]
where \( \pi(t) \) is the probability density function with respect to $t$ function, which means evolving with time. 
A hyperedge \( E \subseteq V \) of size \( k \) is said to \emph{exist} at time \( t \) if every pair of nodes \( i, j \in E \) is connected:
\[
\Pr(E \text{ exists at time } t) = \prod_{i, j \in E, i < j} \Pr((i, j) \in A_t),
\]
where \( A_t \) is the adjacency matrix of the network snapshot at time \( t \).
\end{definition}
\begin{proposition}[Probability of Fully Connected Hyperedge]
For a hyperedge \( e \) with size $k$ , the probability of \( e \) being fully connected at time \( t \) is given by:
\[
\Pr(e \text{ is fully connected at time } t) =
\begin{cases}
\prod_{i, j \in e, i < j} \Lambda_t[a, a], & \text{if } e \subseteq C_a, \\
\prod_{i \in C_a, j \in C_b} \Lambda_t[a, b], & \text{if } e \subseteq C_a \cup C_b,
\end{cases}
\]
where \( \Lambda_t[a, b] = \Lambda_0[a, b] \cdot \pi(t) \) is the connection probability at time \( t \). And $\Lambda_t[a, b]$, $\Lambda_t[a, a]$ can be seen as a deterministic probability distribution over t.
\end{proposition}

\begin{proof}
The probability of \( e \) being fully connected requires that all pairs of nodes in \( e \) are connected. 

\textbf{Case 1: \( e \subseteq C_a \)}. All nodes in \( e \) belong to the same community \( C_a \). The connection probability between any two nodes \( i, j \in e \) is \( \Lambda_t[a, a] \). There are \( \binom{k}{2} \) node pairs, so:
\[
\Pr(e \text{ is fully connected} \mid e \subseteq C_a) = \prod_{i, j \in e, i < j} \Lambda_t[a, a] = \left(\Lambda_t[a, a]\right)^{\binom{k}{2}}.
\]

\textbf{Case 2: \( e \subseteq C_a \cup C_b \)}. Nodes in \( e \) are divided between communities \( C_a \) and \( C_b \), with \( k_a \) nodes in \( C_a \) and \( k_b \) nodes in \( C_b \) (\( k_a + k_b = k \)). The connection probability between nodes in \( C_a \) and \( C_b \) is \( \Lambda_t[a, b] \), and there are \( k_a \cdot k_b \) such pairs. Therefore:
\[
\Pr(e \text{ is fully connected} \mid e \subseteq C_a \cup C_b) = \prod_{i \in C_a, j \in C_b} \Lambda_t[a, b] = \left(\Lambda_t[a, b]\right)^{k_a \cdot k_b}.
\]
This completes the proof.
\end{proof}

Next, we introduce

\begin{lemma}[Lemma 17 in \cite{germain2015risk}]
    For any two distributions \( P \) and \( Q \) defined on \( \mathcal{H} \), and any function \( \psi : \mathcal{H} \to \mathbb{R} \), the following inequality holds:
    \begin{equation}
        \mathbb{E}_{h \sim Q} [\psi(h)] \leq D_{\text{KL}}(Q \| P) + \ln \mathbb{E}_{h \sim P} [e^{\psi(h)}].
    \end{equation}
\end{lemma}

\begin{condition}\label{cond}
For every hyperedge \( E \in \mathcal{E} \), there does not exist a hyperedge \( E' \in \mathcal{E} \) such that \( E \subset E' \).
\end{condition}
\begin{theorem}[Theorem 3 in \cite{wang2024graphs}]\label{thm:3}
Let \( \mathcal{H} = (V, \mathcal{E}) \) be a hypergraph that only satisfies Condition \ref{cond} in \textit{Theorem 1} with \( m = |\mathcal{E}| \). Denote by \( Acc(\mathcal{H}) \) the accuracy of the maximal clique algorithm for reconstructing \( \mathcal{H} \). Then,
\begin{equation}
    \min_{\mathcal{H}} p(\mathcal{H}) \leq 2^{-\left( \lceil m/2 \rceil -1 \right)} \ll 2^{-m}.
\end{equation}
\end{theorem}


Now we are ready to prove the Theorem~\ref{thm:main}.

\begin{theorem-repeat}{thm:main}
Let \( P \) be an exponential distribution with rate parameter \( \lambda \), and let \( Q \) be the probility distribution of $\Lambda_t$. 
Consider a hypergraph \( \mathcal{H} \) with \( m = |\mathcal{E}| \) hyperedges. Denote by \( Acc(\mathcal{H}) \) the accuracy of the maximal clique algorithm for reconstructing \( \mathcal{H} \). Then, the expected accuracy satisfies:
\begin{equation}
\begin{aligned}
    \min_{\mathcal{H}} \mathbb{E}_{\Lambda_t \sim Q} \big[ Acc(\mathcal{H}) \big] 
    \leq  D_{\text{KL}}(Q \| P) -  \textit{Const} \cdot \frac{n(n-1)}{2} \cdot (1 - e^{-\lambda t})^{2K},
\end{aligned}
\end{equation}
where we define the constant
$\textit{Const} = \mathbb{E}_{\Lambda_t \sim P} \sum_{e=1}^{K} \sum_{f=1}^{K} \prod_{i \in C_e, j \in C_f} \Lambda_0[C_e, C_f]$
\end{theorem-repeat}

\begin{proof}
The proof follows from the bound on the accuracy of the maximal clique algorithm, combined with the KL-divergence term and the expectation over the exponential distribution. The term \( (1 - e^{-\lambda t}) \) captures the time-dependent growth in hyperedge formation, which asymptotically converges as \( t \to \infty \).
\begin{equation}
    \begin{aligned}
        \min_{\mathcal{H}} \mathbb{E}_{\Lambda_t \sim Q} \big[ Acc(\mathcal{H}) \big] 
    &\leq \mathbb{E}_{\Lambda_t \sim Q} \left[ 2^{-\left( \lceil m/2 \rceil -1 \right)} \right] 
    \ll \mathbb{E}_{\Lambda_t \sim Q} \left[ 2^{-m} \right] \\
    & \leq D_{\text{KL}}(Q \| P) -  \mathbb{E}_{\Lambda_t \sim P} \left[ m \cdot\ln \left(\frac{e}{2} \right)\right]\\
    &=  D_{\text{KL}}(Q \| P) -  \mathbb{E}_{\Lambda_t \sim P} \sum_{e=1}^{K} \sum_{f=1}^{K} \prod_{i \in C_e, j \in C_f} \Lambda_t[C_e, C_f] \cdot \frac{n(n-1)}{2}\\
        &=  D_{\text{KL}}(Q \| P) -  \textit{Const} \cdot \frac{n(n-1)}{2} \cdot (1 - e^{-\lambda t})^{2K}
    \end{aligned}
\end{equation}
where we define the constant
$\textit{Const} = \mathbb{E}_{\Lambda_t \sim P} \sum_{e=1}^{K} \sum_{f=1}^{K} \prod_{i \in C_e, j \in C_f} \Lambda_0[C_e, C_f]$.
\end{proof}

\section{Proof of Theorem~\ref{expressive}}\label{app:proof2}
First, we provide a definition of the expressive power of the temporal graph neural networks as~\citet{souza2022provably}.
For a single node $n$ in a temporal graph at time $t$, its message-passing procedure can be described as a temporal computation tree~(TCT). The TCT will take the node $n$ as its root, and for any node that is linked to $n$ before $t$ we add them and children nodes to the node $n$. Then for each of the child nodes, we are looking for the nodes linked to it before it is linked to the node $n$. 
By applying this procedure recursively, we can build the $TCT$ for node $n$ with arbitrary depth~(but can not surpass the size of the temporal graphs). 
For two nodes' $u$ and $v$, suppose their $TCT$ are $TCT(u)$ and $TCT(v)$ which are not isomorphic. 
If a TGNN can produce different embeddings for $u$ and $v$, we say it can distinguish $TCT(u)$ and $TCT(v)$, otherwise, the TGNN can not distinguish the two TCTs.
Next, we formally compare the expressive power of the temporal graph neural networks.
\begin{definition}[Expressiveness of TGNN]
    Consider two temporal graph neural networks $TGNN_1$ and $TGNN_2$, we denote the set of all TCT pairs they can distinguish as $\mathcal{T}_1$ and $\mathcal{T}_2$, respectively. Then we say $TGNN_1$ is strictly more powerful than $TGNN_2$ if and only if $\mathcal{T}_1 \subset \mathcal{T}_2$.
\end{definition}

Now we are ready to prove the Theorem~\ref{expressive} by showing $\mathcal{T}_{MP-TGN}\subset\mathcal{T}_{HTGN}$.

We will first prove that all the node pairs that MP-TGN can distinguish can also be distinguished by HTGN. This can be divided into two cases.

\textit{Case 1:} There are no higher-order structures in the TCTs. In this situation, HTGN will be equivalent to the pairwise message-passing temporal graph neural networks. Hence, they will produce the same embeddings.

\textit{Case 2:} There are higher-order structures in the TCTs. For every link $e$ used by the MP-TGN by to compute the embedding of the node $n$'s embedding, HTGN must also have used it to compute the node $n$'s embedding. This is because the node $n$'s embedding is calculated based on the hyperedge embeddings in $n$'s neighborhood. Any link $e$ that is computed by MP-TGN must have been used to compute one of the hyperedge features with injective aggregation function. Therefore, $e$'s information is also used to compute $n$'s embedding by HTCN. Hence, if two TCTs can be distinguished by MP-TGN, they can also be distinguished by HTGN.

Now we have shown that $\mathcal{T}_{MP-TGN}\subseteq\mathcal{T}_{HTGN}$. Next, we will show that HTGN can distinguish TCTs of two nodes which MP-TGN can not. Consider a simple example shown in Figure~\ref{fig:expressive}:

\begin{figure}[ht]
\vskip 0.2in
\begin{center}
\centerline{\includegraphics[width=0.8\columnwidth]{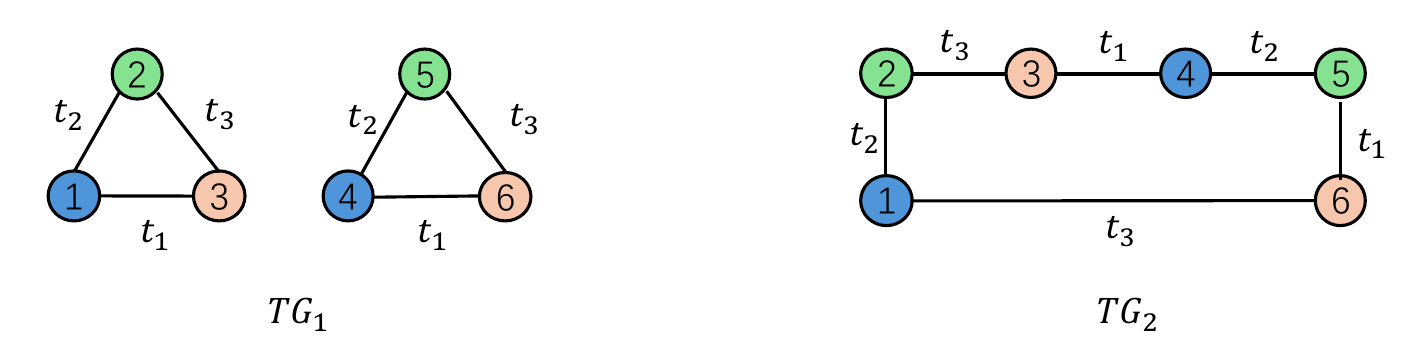}}
\caption{Two temporal graphs which HTGN can distinguish but MP-TGN can not. The colors represent node features.}
\label{fig:expressive}
\end{center}
\vskip -0.2in
\end{figure}

For MP-TGN, the embeddings of the corresponding nodes in the two figures are always the same~\cite{souza2022provably}. Hence, MP-TGN can not distinguish these two temporal graphs.  However, from the perspective of higher-order structures, there are two hyperedges of size 3 in $TG_1$. but none in $TG_2$. The difference can be captured by HTGN. Hence, the HTGN can distinguish the two temporal graphs while MP-TGN can not. We have $\mathcal{T}_{MP-TGN}\subset\mathcal{T}_{HTGN}$ and Theorem~\ref{expressive} is proved.
\section{Dataset Details}\label{app:dataset}

Table~\ref{tab:stats} presents the detailed statistics of datasets we used in our experiments, including the dataset's domain, type, and sizes. It can be observed that the datasets we use span a large range of domains and sizes. Hence, the datasets could serve as a solid foundation for the model evaluation.

\begin{table}[ht]
\caption{Dataset statistics.}
\label{tab:stats}
  \resizebox{\linewidth}{!}{%
  \begin{tabular}{l | l l l l l l l}
  \toprule
  \toprule
  Dataset & Domain& Type & \# Nodes & \# Edges
  & \# Unique Edges & \# Steps  \\ 
  \midrule

  Reddit & Social& Heterogeneous & 10,984 & 672,447 & 78,516 & 669,065  \\ 
  MOOC & Interaction & Heterogeneous & 7,144 & 411,749 & 178,443 & 345,600  \\ 
  LastFM & Interaction & Heterogeneous & 1,980 & 1,293,103 & 154,993 & 1,283,614  \\ 
  Enron & Social& Homogeneous & 184 & 125,235 & 3,125 & 22,632  \\ 
  SocialEvo & Proximity & Homogeneous & 74 & 2,099,519 & 4,486 & 565,932  \\ 
  UCI & Social & Homogeneous & 1,899 &59,835 & 20,296 & 58,911  \\ 
  tgbl-wiki & Social & Heterogeneous& 9,227 & 157,474 & 18,257 & 152,757  \\ 
tgbl-review & Social& Heterogeneous & 352,637 & 4,873,540 & 4,729,124 & 6,865  \\ 
tgbl-coin & Social& Homogeneous & 638,486 & 22,809,486 & 3,429,240 & 1,295,720  \\ 
tgbl-comment & Social& Homogeneous & 184 & 125,235 & 30,146,596 & 22,632  \\ 
tgbl-flight & Social& Homogeneous & 18143 & 67,169,570 & 3,125 & 1,385  \\ 

  \bottomrule
  \bottomrule
  \end{tabular}
  }

\end{table}
\section{Time Encoding}\label{app:time encoding}

Here we give the detailed method to compute the time-encoding function. We follow the previous method~\cite{xu2020inductive} to project the one-dimensional real-valued time to a d-dimensional vector. Specifically, the time encoding function $\phi$ is given by:
\begin{equation}
    \phi(\Delta t) = [\cos(w_1\Delta t+b_1), ..., \cos(w_d\Delta t+b_2)]
\end{equation}
where $w_i$ and $b_i$ are learnable scalar parameters.
\section{Implementation Details}

In this section, we introduction the implementation details of HTGN. We set the initial learning rate to be $10^{-4}$ and apply Adam~\cite{kingma2014adam} to optimize the rate. For the hyperedge aggregation function, we set $\alpha=2$ and $\beta=10^{-4}$. 
The number of hyperedge message passing layers is 2.
We set the batch size $b$ to be 200.
We keep all the edge feature dimension to be 175.
For homogeneous graphs, the maximal number of neighbors for each node is 20.
For structure updated in heterogeneous bipartite graphs, we allow each edge to keep up to 15 distinct nodes. 
\section{The Effect of Reception Field}
Here we study how the difference hops of neighborhood information affect the performance.  We conduct an ablation study on the effects of different hops of neighborhood information when computing the node embeddings on HTGN's performance. The results are shown in Figure~\ref{fig:ablation1}. We notice that increasing the number of hops yields performance improvements, but these gains diminish with further increases. Consequently, to balance performance and efficiency, we typically use 2-hop neighborhood information for node embeddings.

\begin{figure}[h]
\begin{center}
\centerline{\includegraphics[width=\columnwidth]{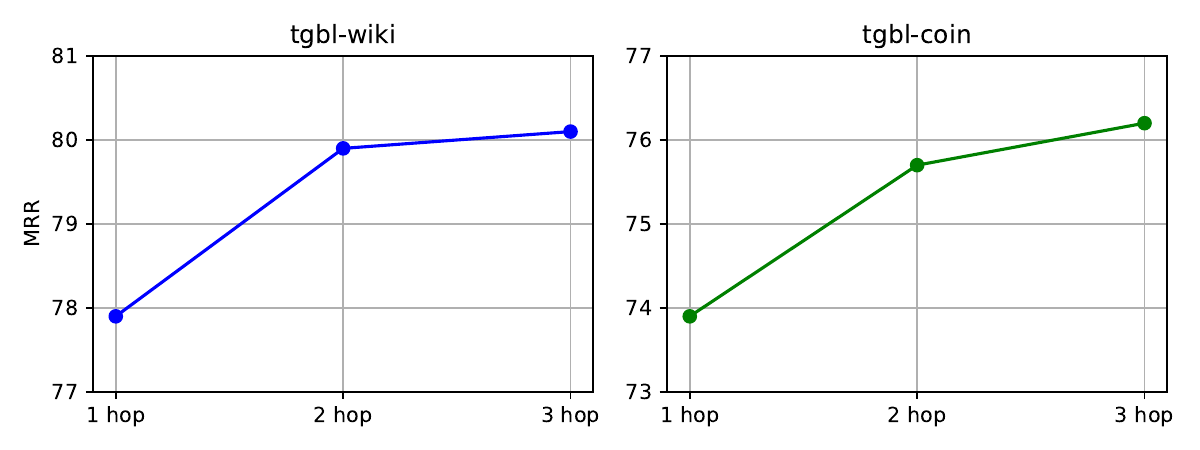}}
\vskip -0.2in
\caption{The effects of model reception field. }
\label{fig:ablation1}
\end{center}
\vskip -0.2in
\end{figure}
\section{An Example of the Formation of False Hyperedges}\label{app:case}
Here we provide an example to illustrate the formation of maximal cliques which are not hyperedges. 
As shown in Figure~\ref{fig:case}, there are three real hyperedges~(marked by different colors) in which nodes interact collectively within a short time slot. 
However, after the interactions, a new maximal clique $\{1,3,8\}$ has been formed. It does not belong to any other maximal cliques, and we can not distinguish it from the real hyperedges if the interaction time information is not given.

\begin{figure}[h]
\begin{center}
\centerline{\includegraphics[width=0.6\columnwidth]{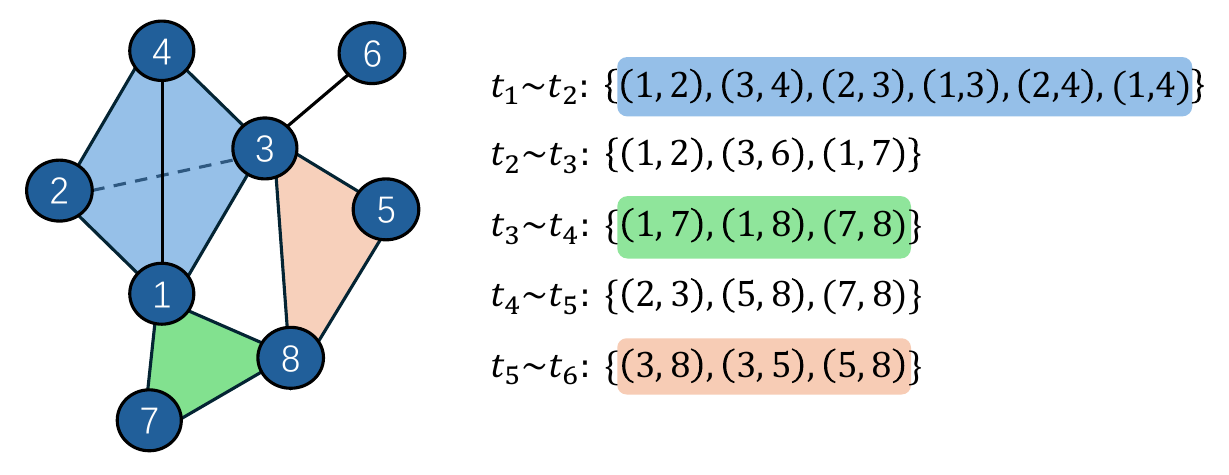}}
\vskip -0.2in
\caption{An example temporal graph. }
\label{fig:case}
\end{center}
\vskip -0.2in
\end{figure}

\newpage
\section{The Hyperedge Algorithm for Heterogeneous Bipartite Graphs}\label{app:bipartite}
In Algorithm~\ref{alg:bipartite}, we show how the hyperedges in heterogeneous bipartite graphs are constructed and updated step by step.

\begin{algorithm}[h]
   \caption{Hyperedge construction for heterogeneous bipartite graphs}
   \label{alg:bipartite}
\begin{algorithmic}[1]

   \STATE {\bfseries Input:} A event-stream source will generate a link $(u,v)$ at every time step $t$. A bipartite graph with two partitions of the nodes $\mathcal{N}_1$ and $\mathcal{N}_2$, where we suppose $|\mathcal{N}_1|>|\mathcal{N}_2|$
   \STATE {\bfseries Initialize:} A set $\mathcal{H}$ of hyperedges, in which every hyperedge corresponds to a node in $\mathcal{N}_2$. For instance, we denote the hyperedge corresponding to node $n$ as $E[n]$.

   \REPEAT
   \STATE Receive a new link $(u, v)$. Suppose $v \in \mathcal{N}_2$ and $u \in \mathcal{N}_1$.
   \STATE $E[v]$.add($u$)
   \IF{The size of $E[v]$ is greater than $b$}
   \STATE Find the node in $|E(v)|$ which has largest time gap since its last interaction with $v$. Denote the node as $o$.
      \IF{The time gap is greater than $t'$}
   \STATE $|E(v)|$.delete($o$)
   \ENDIF
   \ENDIF

   \UNTIL{No new link is generated}
\end{algorithmic}
\end{algorithm}





\end{document}